\documentclass[11pt]{article}

\usepackage{algorithm}
\usepackage{algorithmic}
\usepackage{amsthm}

\usepackage{fullpage}

\usepackage{amsmath,amsfonts}

\usepackage{algorithmic}

\usepackage{hyperref}
\usepackage{graphicx}

\newcommand{\rE}{{\mathbb E}}

\newcommand{\rR}{{\mathbb R}}
\newcommand{\algname}{{Feel-Good}}
\newcommand{\priorname}{{Feel-Good}}
\newcommand{\dc}{\mathrm{dc}}

\newcommand{\cS}{{\mathcal S}}
\newcommand{\cX}{{\mathcal X}}
\newcommand{\cA}{{\mathcal A}}

\newcommand{\cF}{{\mathcal F}}
\newcommand{\reg}{{\mathrm{REG}}}

\newcommand{\BE}{\mathrm{BE}}
\newcommand{\FG}{\mathrm{FG}}
\newcommand{\LS}{\mathrm{LS}}

\newcommand{\cT}{{\mathcal T^\star}}

\newtheorem{assumption}{Assumption}

\title{\algname\ Thompson Sampling for Contextual Bandits and Reinforcement Learning}

\author{Tong Zhang
\thanks{Google Research and the Hong Kong University of Science and Technology}
}

\date{}

\newtheorem{lemma}{Lemma}
\newtheorem{theorem}{Theorem}
\newtheorem{proposition}{Proposition}
\newtheorem{definition}{Definition}

\newcommand{\suppl}{appendix}

\begin{document}

\maketitle

\begin{abstract}
  Thompson Sampling has been widely used for contextual bandit problems due to the flexibility of its modeling power. However, a general theory for this class of methods in the frequentist setting is still lacking.
  In this paper, we present a theoretical analysis of Thompson
  Sampling, with a focus on frequentist regret bounds.
  In this setting, we show that the standard Thompson Sampling is not 
  aggressive enough in exploring new actions, leading to suboptimality in some pessimistic situations.
  A simple modification called
  {\em \algname\ Thompson Sampling}, 
  which favors high reward models more aggressively than the standard Thompson Sampling, is proposed to remedy this problem.
  We show that the theoretical framework can be used to derive 
  Bayesian regret bounds for standard Thompson Sampling, and
  frequentist regret bounds for \algname\ Thompson Sampling.
  It is shown that in both cases, we can reduce  the bandit
  regret problem to online least squares regression estimation.
   For the frequentist analysis, the online least
  squares regression bound can be directly obtained using online
  aggregation techniques which have been well studied.
  The resulting bandit regret bound matches the minimax lower bound in the finite action case.
Moreover, the analysis can be generalized to handle a class of linearly embeddable contextual bandit problems (which generalizes the popular linear contextual bandit model). The obtained result again matches the minimax lower bound.
Finally we illustrate that the analysis can be extended to handle some MDP problems. 
\end{abstract}
\section{Introduction}

This paper considers the contextual bandit problem
\cite{langford2007epoch} as well as its generalization to
contextual reinforcement learning \cite{jiang2017contextual}.
The contextual bandit problem can be regarded as a repeated game between a player (bandit algorithm) and an adversary  as follows: at time $t=1,2,\ldots$
\begin{itemize}
\item The player observes a context $x_t \in \cX$ from the adversary;
\item The player picks an action $a_t \in \cA(x_t)$;
\item The player observes a reward $r_t$.
\end{itemize}
We assume that $\cA(x_t)$ is the set of allowable actions for context $x_t$. 
We also assume that the reward $r_t$ is stochastic which depends only on $(x_t,a_t)$, and 
\begin{equation}
f_*(x_t,a_t) =\rE_{r_t|a_t,x_t} \; r_t  , \label{eq:reward}
\end{equation}
where $f_*$ is an unknown action value function. 
However, we allow an {\em adversarial} opponent who may pick $x_t$ based on the game history 
\[
S_{t-1}= \big[(x_1,a_1,r_1),\ldots, (x_{t-1},a_{t-1},r_{t-1}) \big]
\]
at any time $t$. 
Our goal is to maximize the expected reward 
\[
\rE \; \sum_{s=1}^t f_*(x_s,a_s) ,
\]
where the expectation is with respect to the internal randomization of the algorithm and the randomness in observations.
The maximum expected reward at $x$ is
\[
f_*(x) = \max_{a \in \cA(x)} f_*(x,a) .
\]
The quality of a bandit algorithm is measured by its (frequentist) regret 
\[
\rE \; \sum_{s=1}^t  \reg_s , \qquad \reg_s = \left[ f_*(x_s)
- f_*(x_s,a_s) \right] .
\]
In the theoretical analysis of contextual bandits,
the goal is to obtain regret bounds. 

A particularly important class of algorithms for contextual bandits is
Thompson Sampling \cite{thompson1933likelihood}, which has been widely used in practice with good
empirical performance. However, there is a lack of general frequentist
regret analysis for this class of algorithms.
In the analysis of Thompson Sampling, one often considers another type of regret bound called Bayesian regret. If we assume that
$f_*(x,a)= f(\theta_*,x,a)$ for some $\theta_*$ that is drawn from a
prior distribution $p_0(\theta)$, the Bayesian regret is the averaged frequentist regret
over the prior:
\[
  \rE_{\theta_* \sim p_0} \; \rE \;  \sum_{s=1}^t  \reg_s .
\]

In this paper we develop a theoretical framework to analyze Thompson
Sampling for contextual bandits. The framework introduces a 
decoupling coefficient technique to the reduce regret analysis for
Thompson Sampling into an online least squares estimation problem,
which is in a style similar to \cite{foster2020beyond}. We show this
conversion can be done both for Bayesian regret analysis and for frequentist
regret analysis.
For frequentist regret, we further show that 
an additional exploration term called \algname\ is needed, which
favors models that are optimistic historically.
With this added exploration term, we can employ standard online
aggregation techniques to obtain bounds for the least squares
estimation problem. 
Moreover, it is shown that the analysis can be extended to some settings in reinforcement learning.
Our theoretical framework provides a simple mathematical framework to
analyze Thompson Sampling.

\section{Related Work}
\label{sec:related}

The contextual bandit problem can be regarded as a generalization of multi-armed bandit with side information \cite{langford2007epoch}. It has many practical applications such as online advertising, recommendation systems and mobile health \cite{li2010contextual,agarwal2016making,tewari2017ads}. 
Due to its wide range of applications, there is significant effort in developing algorithms and theoretical analysis for contextual bandit problems. 

In general, contextual bandit algorithms can be characterized into
policy based and value based methods. Policy based methods include
EXP4 \cite{auer2002nonstochastic}, and empirical classification
minimization based methods such as epoch greedy
\cite{langford2007epoch}, and \cite{dudik2011efficient}. However, they
are computationally inefficient for large problems because solving
classification problems can be computationally costly. Moreover, it is
often difficult to generalize such policy algorithms to handle infinite number of actions.

Related to policy based algorithms are value based methods such as various methods for linear bandits \cite{dani2008stochastic,li2010contextual,abbasi2011improved}, which depend on the concept of upper confidence bound (UCB) \cite{auer2002using}. This class of methods involve the solution of linear regression, and can handle infinitely many actions. 
More recently, it was observed that least squares regression oracles (rather than classification oracles) can be used to derive bandit algorithms for finite actions \cite{foster2020beyond,simchi2020bypassing}, and the resulting bounds match the optimal minimax bounds in \cite{agarwal2012contextual} for finite function family. 
While certain infinite action spaces can also be handled
\cite{foster2020adapting}, either the resulting algorithm requires
complex optimization for each test data, or the required structure is
more restrictive than the linear bandit model of
\cite{dani2008stochastic}. The proof technique is different from UCB,
and employs a policy randomization trick of \cite{abe1999associative}
for the derived policy, together with smart variance controls that
depend on the regression oracles. Unfortunately, as noted by the
authors of \cite{foster2020beyond}, their analysis based on policy randomization is difficult to extend to the reinforcement learning setting. 

A very popular class of algorithms for contextual bandits is Thompson Sampling \cite{thompson1933likelihood}, which has been observed to perform well empirically 
\cite{chapelle2011empirical,Osband17-icml}. This class of algorithms
can be regarded as value based, but it employs a different mechanism
than UCB to perform exploration. Thompson Sampling often performs
better than UCB empirically and there are many existing posterior
approximation techniques developed by the Bayesian community that can
be used to sample from the posterior. 
However, its theoretical analysis is rather limited. Although near optimal 
results are known for non-contextual multi-armed bandits 
\cite{agrawal2013further}, it is unclear how well the method works for the general case. 
In fact, even for linear bandits, the results are not optimal \cite{agrawal2013thompson}. 
It is also not known whether Thompson Sampling can achieve the optimal worst case frequentist regret bound for the general contextual bandit problem considered in \cite{foster2020beyond}, although some related results are known for Bayesian regret which averages over a known prior distribution \cite{russo2014learning}.

This paper tries to resolve this open problem. We derive a decoupling
technique which allows us to turn the regret analysis of Thompson Sampling into
online least squares regression bound analysis, in a style motivated
by \cite{foster2020beyond}.
It is shown that it is possible to establish a unified analysis of both
Bayesian regret and frequentist regret for Thompson Sampling. 
For Bayesian regret bounds, the standard Thompson Sampling algorithm is sufficient. However,  for frequentist regret bounds,
we show that the standard Thompson Sampling leads to a suboptimal worst case regret bound. 
To remedy this problem, we introduce an additional term called \priorname\ exploration that encourages optimistic exploration in Thompson Sampling. We show that with this modification,  a frequentist regret bound comparable to that of 
\cite{foster2020beyond} can be obtained for the case of finite function classes with finite action space.
The analysis using the \priorname\ exploration term leads to an exploration mechanism that is different both from UCB and from the policy randomization trick considered by \cite{foster2020beyond}. We note that Thompson Sampling randomizes over value functions (with deterministic greedy policy) instead of over policies as in \cite{foster2020beyond} and EXP4. 
This allows us to generalize the analysis easily to deal with infinite action spaces and reinforcement learning. 

For the case of infinite action space, we introduce a new contextual bandit model called {\em linearly embeddable} bandits, which directly generalizes linear bandits. The model allows a context dependent non-linear embedding of the linear weights, and we show that regret bounds can be obtained for general nonlinear parametric families of embeddings, which match the regret bounds of \cite{dani2008stochastic} for linear bandits. This improves some earlier frequentist regret bounds for Thompson Sampling, which had a suboptimal $O(d^{3/2})$ dependency on the dimensionality $d$ \cite{agrawal2013thompson,abeille2018improved},
to the optimal dependency $O(d)$, matching those of linear UCB style methods  \cite{dani2008stochastic,abbasi2011improved}.
Note that Bayesian regret bounds with $O(d)$ dependency can be
obtained for linear bandits \cite{russo2014learning,russo2016information}.

It is also possible to apply the idea of \priorname\ term and its proof technique
 to the contextual reinforcement learning problem, 
which is a model studied in \cite{jiang2017contextual,CNAAJR18-neurips}. The regret bound we obtain is similar to that of 
\cite{jin2020provably} for linear Markov Decision Process (MDP).
For simplicity, in this work, we only consider the case that MDP
transitions are deterministic, which were investigated by some earlier
work \cite{CNAAJR18-neurips}, and leave the more general case to
future work.
We note that related regret bounds have  been derived for Thompson
Sampling for the tabular MDP case \cite{agrawal2017posterior}, and for
the related method of randomized least squares value iteration
\cite{osband2019deep,zanette2020frequentist,agrawal2020improved}. However,
our results are are different, and allows contextually dependent
nonlinear embeddings of linear MDPs in an adversarial setting.

It is also worth pointing out the general analysis presented in this
paper has its limitations, especially when it is applied to bandit
problems with special structures. For example, even for the simple
case of multi-armed structured bandit problem
\cite{lattimore2014bounded}, regret bounds from this paper may be suboptimal.
The same suboptimality is also present for analysis such as
\cite{foster2020beyond}, which focused on the general nonlinear
 contextual bandits, but failed to obtain bound of the
form $O(\sqrt{d T\log K})$ for $d$ dimensional linear bandits with $K$ arms. That is,
the special structure of linear bandits are not fully utilized in the nonlinear
analysis of \cite{foster2020beyond} and the analysis of the current paper. 
Similarly, results obtained in
this paper do not yield
optimal bounds for some other structured bandit problems such as
latent bandits \cite{maillard2014latent,hong2020latent}.

\section{Thompson Sampling}

In Thompson Sampling, we consider a parameter space $\theta \in
\Omega$, a prior $p_0(\theta)$ and a reward likelihood
with negative log-likelihood $L(\cdot)$
\[
p(r_t|\theta, x_t, a_{t}) \propto \exp( - L(\theta, x_t,a_t,r_t)) .
\]
Each $\theta$ is associated with a function $f(\theta,x,a)$, which is an approximation of the true value function $f_*(x,a)$.
Note that in the Bayesian regret analysis, we assume that the prior and likelihood are both correct, and $f_*(x,a)=f(\theta_*,x,a)$ for some $\theta_*$ drawn from the prior $p_0$. 
In the frequentist regret analysis,  we do not assume that either prior or the likelihood is correct. 

We also define the induced action $a(\theta,x)$, and the value function $f(\theta,x)$ at $x$ according to $\theta$ as follows:
\begin{equation}
        a(\theta,x) \in \arg\max_{a \in \cA(x)} f(\theta,x,a) , \qquad f(\theta,x)=f(\theta,x,a(\theta,x)) = \max_a f(\theta,x,a) . \label{eq:a-determ}
    \end{equation}

In the Bayesian formulation, and assume that the prior is correctly
specified, we can regard the posterior as
\begin{equation}
p(\theta|S_{t-1}) \propto \exp \left(-\sum_{s=1}^{t-1} L(\theta,x_s,a_s,r_s)\right) p_0(\theta) . \label{eq:posterior}
\end{equation}

The Thompson Sampling algorithm does the following at each time step $t$
\begin{itemize}
    \item draw $\theta_t \sim p(\theta|S_{t-1})$;
    \item take action $a_t = a(\theta_t, x_t)$.
\end{itemize}
The resulting algorithm is presented in Algorithm~\ref{alg:ts}. 

\begin{algorithm}[H]
\caption{Thompson Sampling for Contextual Bandits}
\label{alg:ts}
  \begin{algorithmic}[1]
  \FOR{$t=1,2,\ldots,T$}
    \STATE Observe $x_t \in \cX$ 
   \STATE Draw $\theta_t \sim p(\theta|S_{t-1})$ according to \eqref{eq:posterior} 
  \STATE  Let $a_t=a(\theta_t,x_t)$
  \STATE  Observe reward $r_t$
  \ENDFOR
\end{algorithmic}
\end{algorithm}

The distribution of $a_t$ can be obtained by integrating out $\theta_t$ as:
\begin{equation}
p(a_t=\tilde{a}|x_t,S_{t-1})
= \rE_{\theta_t \sim p(\theta|S_{t-1})} I (a(\theta_t,x_t)=\tilde{a}) ,
\label{eq:p-a}
\end{equation}
where $I(\cdot)$ is the indicator function.
Therefore Thompson Sampling is equivalent to sampling $a_t$ according to \eqref{eq:p-a}.
 
\subsection{Suboptimality of Frequentist Regret for Thompson Sampling}
\label{sec:subopt}

In standard Thompson Sampling, a natural choice is to pick the likelihood as
\begin{equation}
L(\theta,x,a,r) = 
\eta (f(\theta,x,a)-r)^2
\label{eq:standard}
\end{equation}
for some appropriate $\eta>0$.  
In the Bayesian setting, this corresponds to a stochastic Gaussian
likelihood reward model with variance $1/(2\eta)$.
Theoretically, the benefit of using a Gaussian model is that
its concentration and anti-concentration properties are well
understood, which are useful for regret analysis.
However, in the frequentist setting, this likelihood model can also be used for non-Gaussian reward problems because we do not assume that the Bayesian model is correct. 

To understand the behavior of Thompson Sampling, we are particularly interested in the case considered in  \cite{foster2020beyond}, 
which showed the following.
Assume that the function class $\cF=\{f(\cdot)\}$ contains $N$ members, and $f_* \in \cF$. Moreover, assume that
the action space $\cA$ is finite with $|\cA|=K$, then there exists a contextual bandit algorithm which achieves a frequentist regret bound of 
\[
\sum_{t=1}^T \rE \; \reg_t = O\left(\sqrt{K T \ln N}\right)
\]
after $T$ steps. This matches the lower bound in
\cite{agarwal2012contextual}.
It is natural to ask whether similar bounds can be obtained for Thompson Sampling in the frequentist setting.  

In the following, we show that the standard implementation of Thompson Sampling in \eqref{eq:standard}
leads to suboptimal regret in the worst case, which  motivates the \algname\ Thompson Sampling method introduced in this paper.

We consider two actions $a \in \{1,2\}$, and a function class with $N$ members
$\Omega=\{\theta_1,\ldots,\theta_N\}$. Moreover, we assume that $\theta_*=\theta_1$ is the correct reward model:
\[
\rE_{r|x,a} r = f(\theta_*,x,a) = 
\begin{cases}
1 & \text{if } a = 2 \\
0.5 & \text{if } a = 1 .
\end{cases}
\]
Assume further that for all $j \geq 2$:
\[
  f(\theta_j, x,a) =
  \begin{cases}
0.4 j/N & \text{if } a = 2 \\
0.5 & \text{if } a = 1 .
  \end{cases}
\]
Let $p_0$ be the uniform distribution on $\Omega$, so that each $\theta_j$ has a probability of $1/N$ at the beginning.

\begin{proposition}
  Given any $T \geq 1$, we have the following lower bound on regret for standard Thompson Sampling of \eqref{eq:standard}
  \[
    \sum_{t=1}^T \rE \; \reg_t \geq 0.5 T (1-1/N)^T .
  \]
\end{proposition}
\begin{proof}
  At the first step $t=1$, without any information, we pick $\theta_j$ with $j>1$ with probability $1-1/N$. This means that we will choose the greedy policy $a=1$ in Thompson sampling associated with $\theta_j$ when $j>1$.
  Since $f(\theta_j,1,x)=0.5$ for all $j$, we have no information to differentiate any $\theta_j \in \Omega$, and thus the posterior remains uniform over $\Omega$. This can only change when we choose $j=1$ at some time $T$. 
It follows that the probability of 
sampling $\theta_j$ with $j \geq 2$ for all $t \leq T$ is $(1-1/N)^T$. Each takes the suboptimal action $a=1$, and suffers a regret of $0.5$.
We thus obtain the desired bound.
\end{proof}

The result implies that at time $T=N$, we have a frequentist regret
bound of $\Omega(T)$, which is linear in $T$.
It follows that the frequentist regret of standard Thompson Sampling
is suboptimal, compared to the regret bound of $O(\sqrt{T \ln T})$
achieved in
\cite{foster2020beyond}.

\subsection{\algname\ Thompson Sampling}
To overcome the difficulty of the standard Thompson Sampling, we propose the addition of an  exploration term by favoring $\theta$ with larger $f(\theta,x)$. Specifically, we take
\begin{equation}
L(\theta,x,a,r) = 
\eta (f(\theta,x,a)-r)^2
- \lambda  \min(b,f(\theta,x))  \label{eq:loss}
\end{equation}
for some constant $b$ in the Thompson Sampling algorithm,
where $\lambda \geq 0$ is a tuning parameter. The Standard Thompson
Sampling in \eqref{eq:standard} is equivalent to the case of
$\lambda=0$.

The additional exploration term $f(\theta,x)$ encourages the method to choose a model $\theta$ with a large maximum reward on historic observations. Such a choice  favors a model $\theta$ with a large historic maximum reward, which are model parameters that feel good based on the history. This term can be regarded as a data dependent exploration term, which we call
 {\em \priorname\ exploration}, and the resulting Thompson Sampling algorithm is referred to as {\em \algname\ Thompson Sampling}.

In the example we presented, where the standard Thompson Sampling
method is suboptimal, we note that the \algname\ sampling formulation
will favor the choice of $\theta_*$. This is because $f(\theta_*,x,2)$
is a larger reward than alternatives by a constant margin. A simple
calculation suggests that with $\lambda = 1/\sqrt{T}$, we will
choose the optimal $\theta_*$ after $O(\sqrt{T})$
time steps. This leads to a regret of $O(\sqrt{T})$. The example also
suggests that it is better to choose large $\lambda$ in the beginning,
and let it decay to $0$. This can lead to $O(1)$
regret for this example. However, for simplicity, we do not consider
the method of time-varying $\lambda$ in the theoretical analysis of
this paper. 

This observation implies that the standard Thompson sampling method is not aggressive enough in selecting optimistic models, and the additional \priorname\ prior remedies the problem. 
The resulting \algname\ Thompson Sampling method may be regarded as an implementation of the general {\em optimism in the face of uncertainty} principle, and the  \priorname\ prior can be regarded as an analogy of upper confidence bound (UCB) for posterior sampling methods. 
As we will show, this optimism will lead to a provably good regret bound for the general situation which matches (and generalizes) the result of
\cite{foster2020beyond}. Similarly a direct application of linear
bandit bounds in \cite{dani2008stochastic,abbasi2011improved} to multi-armed bandits also leads to suboptimal regret,
because it does not consider the special structure. 
It is worth pointing out that for the general contextual bandit problem consider here, the randomized least squares approach considered in \cite{osband2019deep} does not lead to sufficient exploration either. This because a perturbation of historic data does not remedy the flat posterior problem in the example of Section~\ref{sec:subopt}.
 
Computationally, with the addition of the \priorname\ exploration
term, one has to reply on approximate MCMC inference
methods to sample from the posterior distribution. While
Section~\ref{sec:simulation} shows that this can be done in practice,
we note that for some simple problems, the standard Thompson Sampling
may take advantage of distribution conjugacy, with closed form
posterior distribution that is easier to sample. For complex problems
where approximate MCMC inference methods are needed, the difference
may not be significant. 

\section{Theoretical Analysis}

This section derives a general regret bound for the \algname\ Thompson Sampling method. For simplicity, we make the following boundedness assumption on the reward. 
\begin{assumption}
  The reward is sub-Gaussian:
  \[
    \ln \rE_{r_t|x_t,a_t} \exp(\rho (r_t - f_*(x_t,a_t))) \leq \frac{\rho^2}{8} .
  \]
  Moreover, we assume that for all $x \in \cX$ and $a \in \cA(x)$:
  $f_*(x,a) \in [0,1]$.
  \label{assump:reward}
\end{assumption}
Note that if we assume that the observed reward $r_t \in [0,1]$, then
Assumption~\ref{assump:reward} holds. This is the situation we are
mostly interested in. 
The sub-Gaussian assumption also holds
with Gaussian noise of variance no more than $0.25$, which is needed
to analyze Bayesian regret with Gaussian likelihood. 
 
Our analysis follows the basic technique of online aggregation methods, such as \cite{vovk2001competitive,kivinen1999averaging,freund1997decision}. This technique was used in the analysis of Bayesian model averaging \cite{yang1999information}, which is closely related to Thompson Sampling (with the only difference of averaging over instead of sampling from the posterior distribution). It was
also employed in the analysis of EXP4 bandit algorithm \cite{auer2002nonstochastic}, which can be regarded as the partial information counterpart of its full information analog Hedge in \cite{freund1997decision}.
Note that both Hedge and EXP4 sample from the posterior, and thus
their theoretical analysis is related to ours. However, unlike
Thompson Sampling considered in this paper, both Hedge and EXP4
employed exponents that are not continuous. Therefore they are
difficult to implement efficiently. In comparison, MCMC methods
such as SGLD \cite{welling2011bayesian} can be employed for \algname\
Thompson Sampling,
as demonstrated in Section~\ref{sec:simulation}.

In order to analyze Thompson Sampling, we have to introduce
new ideas in addition to online aggregation. Define for $b \geq 1$, the truncated function value
\[
f_{b}(\theta,x,a)= \max(-b,\min(b,f(\theta,x,a))) , \quad
f_{b}(\theta,x) =f_{b}(\theta,x,a(\theta,x)) .
\]
Observe that $f_{b}(\theta_t,x_t,a(\theta_t,x_t))=f_{b}(\theta_t,x_t)$.
The starting point of our analysis
is the following decomposition
of the regret at time $t$ as
\begin{align}
\reg_t
=&
 \underbrace{[f_b(\theta_t,x_t,a(\theta_t,x_t))-f_*(x_t,a(\theta_t,x_t))]
}_{\BE_t} - \underbrace{[f_b(\theta_t,x_t)-f_*(x_t)]}_{\FG_t} . \label{eq:regret-decomp}
\end{align}
On the right hand side, the first term is often referred to as the Bellman error in the reinforcement learning literature, which needs to be controlled. 
The second term is the \priorname\ exploration term.

The key technique to control the term $\BE_t$ is based on the
decoupling of the Thompson sampling action choice
$a_t=a(\theta_t,x_t)$ from $\theta_t$. We introduce the following
definition, which can be used to control the first term, and can be
used to quantify the complexity of
exploration in Thompson Sampling.
The key motivation of this definition is to convert the Bellman error with
respect to the action taken by the current policy (no exploration) to least squares error
with independently sampled actions (in such case the exploration is
automatically achieved by independent sampling). It plays the same
role as what UCB does in the traditional bandit
analysis. Conceptually the definition is also related to the idea of information ratio
studied by \cite{russo2014learning,lattimore2021mirror}, which may be
regarded as another way to handle exploration. 

\begin{definition}[Decoupling Coefficient]
  Let $B$ be a contextual bandit with value function $f_*(\cdot)$.
  Given any $x \in \cX$, $b \geq 1$, and $\Omega' \subset \Omega$, 
  we define $\dc(x,b,\Omega',B)$ as the smallest quantity $K$ so that for 
all probability distributions $q(\theta)$ on 
$\Omega'\subset \Omega$, and the induced random policy $\pi_q(\tilde{a}|x)=\rE_{\theta \sim q(\theta)} I(a(\theta,x)=\tilde{a})$ on $\cA(x)$, the following inequality holds
\begin{align*}
 \rE_{\theta \sim q(\theta)}
 &[f_b(\theta,x,a(\theta,x))-f_*(x,a(\theta,x))]\\
 \leq& \inf_{\mu>0} \left[\mu \rE_{\tilde{a} \sim \pi_q(\tilde{a}|x)} 
 \rE_{\theta \sim q(\theta)}
 (f_b(\theta,x,\tilde{a})-f_*(x,\tilde{a}))^2 
 + \frac{K}{4\mu} \right].
 \end{align*}
 \label{def:dc}
\end{definition}
In this paper, we only consider decoupling coefficient that is
independent of $\mu$. More generally, we may also allow $\dc(x,b,\Omega,B)$ to depend on $\mu$.
Using \eqref{eq:p-a}, we obtain the following inequality for all $\mu>0$:
 \begin{align}
   \rE_{\theta_t \sim p(\theta_t|S_{t-1})}
     \BE_t
 \leq& \frac{\dc(x_t,b,\Omega,B)}{4\mu} +\mu \; \rE_{a_t \sim p(a_t|x_t,S_{t-1})} \;
 \rE_{\tilde{\theta} \sim p(\tilde{\theta}|S_{t-1})} \LS_t  , \label{eq:dc}\\
 \text{where } &\qquad \LS_t = (f_b(\tilde{\theta},x_t,a_t)-f_*(x_t,a_t))^2 . \nonumber
 \end{align}
Note that on the left hand side, the action $a(\theta_t,x_t)$ depends on 
$\theta_t$, but on the right hand side, $a_t$ and $\tilde{\theta}$ are
drawn independently from their respective posterior distributions.
Armed with
\eqref{eq:dc}, we can bound the regret for Thompson Sampling by least
squares loss as follows. 
\begin{equation}
 {\rE \; \reg_t \leq \frac{\dc(x_t,b,\Omega,B)}{4\mu} + \mu \rE \; \LS_t -
   \rE \; \FG_t } .
 \label{eq:bandit-regret}
\end{equation}
The term $\rE \; \LS_t - \rE \; \FG_t$  can be bounded using the standard techniques in the analysis of online aggregation algorithms.
The following lemma shows that $\dc(x,b,\Omega,B)$ is upper bounded by
the number of actions, which corresponds to the situation considered
in \cite{foster2020beyond}.
 \begin{lemma}
 Assume that $|\cA(x)| \leq K$  for some $x \in \cX$. Then
 for any $b \geq 1$,  $\dc(x,b,\Omega,B) \leq K$. 
 \label{lem:decouple}
 \end{lemma}
\begin{proof}
Consider any $q(\theta)$, $x \in \cX$, and $\mu>0$.
For any $a \in \cA(x)$, let $p_a=\pi_q(a|x)$.   We have
 \begin{align*}
 &\rE_{\theta \sim q(\theta)}
 I(a(\theta,x)=a)
 |f_b(\theta,x,a)-f_*(x_t,a)|\\
=&\rE_{\theta \sim q(\theta)} \frac{I(a(\theta,x)=a)}{(2\mu p_a)^{1/2}} \cdot (2\mu p_a)^{1/2}
 |f_b(\theta,x,a)-f_*(x,a)|\\
\stackrel{(a)}{\leq}&\rE_{\theta \sim q(\theta)} \frac{I(a(\theta,x)=a)}{2 \cdot 2\mu p_a} +
 \rE_{\theta \sim q(\theta)} \mu p_a
 (f_b(\theta,x,a)-f_*(x,a))^2 \\
\stackrel{(b)}{=}& \frac{1}{4\mu} +
 \rE_{\theta \sim q(\theta)} \mu p_a
 (f_b(\theta,x,a)-f(\theta_*,x,a))^2 ,
 \end{align*}
 where $(a)$ follows from the algebraic inequality $z_1 \cdot z_2 \leq
 0.5 z_1^2 + 0.5 z_2^2$, and
 $(b)$ follows from the definition of $p_a$. 
 By summing over $a \in \cA(x_t)$, we obtain 
 \begin{align*}
 \rE_{\theta \sim q(\theta)}
& |f_b(\theta,x,a(\theta,x))-f_*(x,a(\theta,x))|\\
 \leq& \frac{K}{4\mu} +\mu \; \rE_{a \sim \pi_q(a|x)} \;
 \rE_{\theta \sim q(\theta)}
 (f_b(\theta,x,a)-f_*(x,a))^2 .
 \end{align*}
This leads to
the desired inequality in Definition~\ref{def:dc}.
\end{proof}

We can now obtain the following general Bayesian regret bound for the standard
Thompson Sampling as follows.
\begin{proposition}
  Consider 
  Algorithm~\ref{alg:ts} with posterior model \eqref{eq:posterior} and
  likelihood \eqref{eq:standard}.
  Let $\dc(\Omega,B)=\sup_x \dc(x,1,\Omega,B)$. Assume the Bayesian model is
  correct,  then the Bayesian regret can be bounded as:
  \[
    \rE_{\theta_* \sim p_0} \; \rE \sum_{t=1}^T  \reg_t \leq \sqrt{\dc(\Omega,B) T \Delta_T} ,
  \]
  where
  \[
    \Delta_T = \rE_{\theta_* \sim p_0} \rE \; \sum_{t=1}^T \rE_{\tilde{\theta} \sim p(\cdot |S_{t-1})} \big(f(\tilde{\theta} ,x_t,a_t)- f(\theta_*,x_t,a_t)\big)^2 .
  \]
  \label{prop:bayes-regret}
\end{proposition}
\begin{proof}
  We note that in the Bayesian setting, all $\theta$ are realizable.
  This implies that $f(\theta,x,a) \in [0,1]$, and $f_b(\theta,x,a)=f(\theta,x,a)$.
Moreover,  the marginal of $p(\theta_t|S_{t-1})$, averaged over
  $S_{t-1}$, is $p_0(\theta_t)$. Therefore in
  \eqref{eq:bandit-regret},
  \[
    \rE \;\FG_t = \rE \;
    [ f(\theta_t,x_t)-f(\theta_*,x_t) ] = \rE_{x_t} \left[ \rE_{\theta_t \sim
      p_0} f(\theta_t,x_t) - \rE_{\theta_* \sim p_0} f(\theta_*,x_t)
  \right] = 0 .
\]
Since the model is correct, we know that $f(\theta,x,a) =
f_b(\theta,x,a)$ with $b=1$ for
all $\theta$, $x$, and $a$. 
We thus have
\[
  \rE \;\reg_t \leq \frac{\dc(\Omega,B)}{4\mu} + \mu \rE \; \LS_t .
\]
By summing over $t=1$ to $t=T$, and then optimizing over $\mu$ on the right hand side, we obtain the desired bound. 
\end{proof}

The result implies that we can essentially obtain a Bayesian regret
bound for Thompson Sampling from Bayesian online least squares regression
bound, and such a result 
is analogous to \cite{foster2020beyond}.
For frequentist regret, both $\rE \; \FG_t$ and $\rE \; \LS_t$ can be 
further bounded using online aggregation techniques.
This leads to the following result, which is a special case of Theorem~\ref{thm:bandit}.

\begin{theorem}
  Consider \algname\ Thompson Sampling in Algorithm~\ref{alg:ts} with posterior \eqref{eq:posterior} and Likelihood
  \eqref{eq:loss}.
Under Assumption~\ref{assump:reward}.  
Let $b=1$ and $\eta \leq 0.25$.
If $|\cA(x_t)| \leq K$ for all $x_t$, then we have the following expected regret bound.
\begin{align*}
\sum_{t=1}^T \rE  \; \reg_t
\leq&
\frac{\lambda KT }{\eta} 
      + (0.25 \eta/\lambda) \sum_{t=1}^T \rE \; \LS_t - \sum_{t=1}^T \rE \; \FG_t \\
\leq&  \frac{\lambda K T}{\eta}
+6 \lambda T
- \frac{1}{\lambda} \rE_{S_t} \ln \rE_{\tilde{\theta} \sim p_0}\exp \left(-
      \sum_{s=1}^t \Delta L(\tilde{\theta},x_s,a_s,r_s)\right) ,
\end{align*}
where
\[
\Delta L(\theta,x,a,r)=\eta [(f(\theta,x,a)-r)^2- (f_*(x,a)-r)^2] -
                           \lambda [\min(b,f(\theta,x))-f_*(x)] .
                         \]
                         \label{thm:bandit-finite-action}
\end{theorem}

In Theorem~\ref{thm:bandit-finite-action}, the least squares
regression loss is further bounded by the log partition
function. As we will show in Section~\ref{sec:examples}, the latter can
be easily estimated for various problems.
While it is possible to
establish a similar bound for the cumulative least square regret loss
$\Delta_t$ in Proposition~\ref{prop:bayes-regret} in various cases,
the proof technique will be more involved. This is because in the frequentist
setting, we are allowed to use a small $\eta$, and apply techniques
from online aggregation, while in the Bayesian
setting, we have to set $\eta=1/2\sigma^2$, where $\sigma^2$ is the
variance of the reward. This learning rate is not sufficiently small
to use online aggregation techniques. Therefore more specialized
analysis is needed.
Since this paper focuses on the simpler online aggregation technique
for bounding the least squares loss, we will not derive bounds for
Bayesian least squares regression. 
Nevertheless, the analogy of Proposition~\ref{prop:bayes-regret} and
Theorem~\ref{thm:bandit-finite-action} demonstrates the fact that 
the \priorname\ exploration is not needed in the Bayesian regret
analysis, although it is crucial in the frequentist regret analysis.

We can further extend the analysis to handle certain infinite action spaces with 
the following linearly embeddable contextual bandit model 
\begin{equation}
f(\theta,x,a) = w(\theta,x)^\top \phi(x,a) , \qquad
f_*(x,a) = w_*(x)^\top \phi(x,a) , \label{eq:bandit-linear}
\end{equation}
where $\phi(x,a) \in \rR^K$ and $w(\theta,x) \in \rR^K$ are known functions. This model is a generalization of contextual bandits with finite actions, and contextual bandits with linear payoff functions.
The key of this model is the separation of parameter $\theta$ and action $a$, so that possibly infinite number of actions can all be embedded into a $K$-dimensional linear space. The definition also resembles the definition of Bellman factorization and Bellman rank in \cite{jiang2017contextual} for contextual MDPs. However, the definition of Bellman rank in \cite{jiang2017contextual}, when applied to contextual bandits, leads to an embedding dimension of $1$, and it does not handle infinite actions directly. In comparison, the modified factorization in 
\eqref{eq:bandit-linear} may be regarded as a context and action dependent version of the Bellman factorization, and hence we may refer to its embedding dimension $K$ as the context-action dependent Bellman rank for contextual bandits.

The following result generalizes Lemma~\ref{lem:decouple} for linearly embeddable contextual bandits. 
The proof can be found in Appendix~\ref{apx:proof-decouple-linear}.
\begin{lemma}
  Assume that \eqref{eq:bandit-linear} holds.
  For any $b \geq 1$. If $f(\theta,x,a) \geq -b$ for all $\theta \in
  \Omega$, $x \in \cX$, and $a \in \cA(x)$, then
 $\dc(x,b,\Omega,B) \leq K$.
 \label{lem:decouple-linear}
 \end{lemma} 
 
Note that the result assumes that $f(\theta,x,a) \geq -b$ for all
$\theta \in \Omega$. In some cases, this condition holds. 
However, in the general case of  potentially misspecified
prior, this condition may not hold for all $\theta \in \Omega$.
Since we know the true reward $f_*(x,a) \geq -b$, it is possible to check this condition
in the Thompson Sampling algorithm, and force the posterior to be the
set that satisfies the condition.
In such case, we may consider the following generalized posterior:
\begin{equation}
p(\theta_t |x_t, S_{t-1}) \propto \exp \left(-\sum_{s=1}^{t-1}
  L(\theta_t,x_s,a_s,r_s)\right)  I(\theta \in \Omega_t)
 p_0(\theta_t) , \label{eq:posterior2}
\end{equation}
where $\Omega_t \subset \Omega$ may depend on both $x_t$ and
$S_{t-1}$.
In order to apply Lemma~\ref{lem:decouple-linear}, we are particularly
interested in the choice of
\begin{equation}
\Omega_t = \big\{ \theta \in \Omega:
    \forall s \leq t, a \in \cA(x_s), f(\theta,x_s,a)  \geq -b\big\} .
\label{eq:omega}
\end{equation}

Note that
\eqref{eq:posterior2} becomes \eqref{eq:posterior} when 
$\Omega_t=\Omega$. Therefore we can focus on the generalized Thompson
Sampling algorithm, with $\theta_t \sim p(\theta|S_{t-1})$ of
\eqref{eq:posterior}
replaced by \eqref{eq:posterior2} in the theoretical analysis.
In practice, we may only need to use the standard choice
of \eqref{eq:posterior} instead of the more complex
\eqref{eq:posterior2}. In fact it might be possible that the condition
$f(\theta,x,a) \geq -b$ can be relaxed with a more careful analysis (for example, this
condition is not required in Lemma~\ref{lem:decouple}). If this is the
case, then \eqref{eq:posterior2} is not necessary.
\begin{theorem}
  Consider the \algname\ Thompson Sampling in Algorithm~\ref{alg:ts}
  with posterior replaced by \eqref{eq:posterior2} and Likelihood
  \eqref{eq:loss}. Assume that $\Omega_T \subset \Omega_{T-1} \cdots
  \subset \Omega_1$. 
Under Assumption~\ref{assump:reward}.  
For any $b \geq 1$, let $\eta \leq 1/(b+1)^2$.
Then we have the following expected regret bound for $b \geq 1$:
\begin{align*}
\sum_{t=1}^T \rE  \; \reg_t
\leq&
\frac{\lambda }{\eta} \sum_{t=1}^T \dc(x_t,b,\Omega_t,B)
      + (0.25 \eta/\lambda) \sum_{t=1}^T \rE \; \LS_t - \sum_{t=1}^T \rE \; \FG_t \\
\leq&  \frac{\lambda}{\eta}\sum_{t=1}^T \dc(x_t,b,\Omega_t,B)
+1.5 \lambda (b+1)^2 T
- \frac{Z_T }{\lambda} ,
\end{align*}
where
\begin{align*}
&Z_t  
  = \rE_{S_t} \ln \rE_{\tilde{\theta} \sim p_0} I(\tilde{\theta} \in
                \Omega_t) \exp \left(- \sum_{s=1}^t \Delta L(\tilde{\theta},x_s,a_s,r_s)\right)\\
  &\Delta L(\theta,x,a,r)=\eta [(f(\theta,x,a)-r)^2- (f_*(x,a)-r)^2] -
                           \lambda [\min(b,f(\theta,x))-f_*(x)] . \nonumber
\end{align*}
\label{thm:bandit}
\end{theorem}

Theorem~\ref{thm:bandit} 
holds for contextual bandits with linearly embeddable
payoffs in \eqref{eq:bandit-linear}, which allows infinitely many arms.
With $\Omega_t$ defined in \eqref{eq:omega}, we obtain from
Lemma~\ref{lem:decouple-linear}
\[
\sum_{t=1}^T \dc(x_t,b,\Omega_t,B) \leq   K T .
\]
We may pick $b=1$ and $\eta=0.25$ 
in Theorem~\ref{thm:bandit}.
This implies the following bound
\begin{equation}
  \sum_{t=1}^T \rE  \; \reg_t
  \leq 4 (K+2) T \lambda - \frac{Z_T}{\lambda} .
  \label{eq:regret-bound}
\end{equation}

Note that since $Z_T$ is a constant (ignoring logarithmic factor) for
parametric models, one can set $\lambda = O(\sqrt{1/T})$ to obtain a
$O(\sqrt{T})$ regret. Some detailed examples are presented below. 

\section{Examples}
\label{sec:examples}
We  assume that the optimal value function $f_*(a,x)$ can be well approximated  within $\Omega$.
That is, there is $\theta_* \in \Omega$ so that the model is nearly correctly specified in that
\begin{equation}
\max_{x,a} |f(\theta_*,x,a) - f_*(x,a)| \leq \delta \label{eq:approx-realizability}
\end{equation}
for some small $\delta \in [0,0.5]$.

\subsection{Finite function class}

We set $b=1$, $\eta=0.25$, and
\[
\lambda = \frac{\delta'}{\sqrt{K+2}} + \sqrt{\frac{\ln N}{4(K+2) T}} 
\]
for some $\delta'>0$.
From \eqref{eq:approx-realizability}, with some algebraic
manipulations, we know that
\begin{align*}
 & \rE \; \sum_{s=1}^T -\Delta L(\theta_*,x_s,a_s,r_s) \\
  =&
\rE\; \sum_{s=1}^T     [ - \eta (f(\theta_*,x_s,a_s)-f_*(x_s,a_s))^2+
                                                         \lambda (\min(b,f(\theta_*,x))-f_*(x))]                                                       
                                                            \\
 \geq& -\eta T \delta^2 - \lambda \delta T  
\geq -(\delta+\lambda) \delta T .
\end{align*}
Moreover, we have
$f(\theta_*,x,a) \geq -0.5 \geq -b$ for all $x$ and $a$.

Assume we have $|\Omega|=N$, with prior $1/N$ on each function.
Since $f(\theta_*,x,a) \geq -1$, we know that
\begin{align*}
Z_T \geq 
\rE \; \ln \frac1N \exp\left(-\sum_{s=1}^T \Delta
  L(\theta_*,x_s,a_s,r_s)\right)
\geq 
-(\delta+\lambda)\delta T - \ln N .
\end{align*}

We obtain from \eqref{eq:regret-bound} that
\begin{align*}
  \rE \; \sum_{t=1}^T \reg_t
  \leq& 4 (K+2) T \lambda - \frac{Z_T}{\lambda} \\
  \leq& 4 \sqrt{(K+2) T \ln N}  + 4\big(1+{\delta'/\delta}+{\delta/\delta'}\big)
        \sqrt{(K+2)} \delta T .
\end{align*}

 This result matches that of 
\cite{foster2020beyond}, and thus matches the lower bound for the
finite action case.
Moreover, our result handles infinite action space naturally as long
as they are linearly embeddable.
Some related results can be found in \cite{foster2020adapting}.

 \subsection{Parametric Function Class}
 \label{sec:parametric}
 
 We assume that \eqref{eq:bandit-linear} holds with $\theta \in \rR^d$, and assume that we have a prior on $\rR^d$ with density function $p_0(\theta)$. 
 We again take $b=1$ and take $\eta=0.25$.

 Assume that both $\ln p_0(\theta)$ and $f(\theta,x,a)$ are Lipschitz
 around $\theta_*$. There exists a constant $\gamma>0$ so that for a
 sufficiently large $T$, if we let
$B_T=\{\theta \in \rR^d: \|\theta-\theta_*\|_2 \leq 1/T\}$, then
$\forall \theta \in B_T$:
 \[
 \ln p_0(\theta) \geq \ln p_0(\theta_*)-1 , \qquad \sup_{x,a}
 |f(\theta,x,a)-f(\theta_*,x,a)| \leq \gamma /T \leq 0.5 .
\]
We obtain from \eqref{eq:approx-realizability}  that for all $\theta \in B_T$, 
\[
  |f(\theta,x,a)-f_*(x,a)| \leq \delta + \gamma/T \leq 1 .
\]
It follows that for all $\theta \in B_T$, $f(\theta,x,a) \geq -1$, and
thus $B_T \subset \Omega_T$.

Moreover,  we have for all $\theta \in B_T$:
 \begin{align*}
   -\Delta L(\theta,x_s,a_s,r_s) 
   \geq &
                                        - \eta
     (f(\theta,x_s,a_s)-f_*(x_s,a_s))^2    -                                                     \lambda |\min(b,f(\theta,x))-f_*(x)|     \\
   & \qquad -
                                        2\eta
                                        |f(\theta,x_s,a_s)-f_*(x_s,a_s)|
                                         \cdot |r_s-f_*(x_s,a_s)|\\
\geq&                                        - (1+\lambda +2 |r_s-f_*(x_s,a_s)|)  (\delta+\gamma/T) .
 \end{align*}
 Note that the sub-Gaussian noise condition in Assumption~\ref{assump:reward} implies that $\rE \;
 |r_s-f_*(x_s,a_s)| \leq 0.5$, therefore
 \[
 \rE \inf_{\theta \in B_t} -\Delta L(\theta,x_s,a_s,r_s) \geq - (2+\lambda) (\delta+\gamma/T) .
 \]
 This implies that
\begin{align*}
Z_T = & \rE \ln \rE_{\theta \sim p_0} \exp \left(-\sum_{s=1}^T \Delta L(\theta,x_s,a_s,r_s)\right)\\
\geq & \rE \ln\left[ p_0(B_T) \inf_{\theta \in B} \exp(-\sum_{s=1}^T \Delta L(\theta,x_s,a_s,r_s))\right] \\
\geq & \ln p_0(B_T)  -(2+\lambda) (\delta T + \gamma) \\
\geq & - 1 + \ln p_0(\theta_*) - d \ln (d T) - (2+\lambda)(\delta T + \gamma) .
\end{align*}
  
Now by setting
\[
\lambda = \sqrt{\frac{\delta'}{K+2}} + \sqrt{\frac{d \ln d T}{4(K+2) T}} 
\]
for some $\delta'>0$,
we obtain the following result:
\begin{align*}
  &\sum_{s=1}^T \rE \; \reg_t
    \leq 4 (K+2) T \lambda - \frac{Z_T}{\lambda} \\
  \leq& 4\sqrt{(K+2)d T \ln (d T)} 
             + 2 (1-\ln p_0(\theta_*)+2\gamma) 
      \sqrt{\frac{(K+2) T}{d \ln (d T)}} 
  \\
   &+   4(\sqrt{\delta/\delta'}+\sqrt{\delta'/\delta})\sqrt{(K+2)\delta}
      T  + (\delta T + \gamma) .
\end{align*}

Consider the Gaussian prior case with
\[
  p_0(\theta_*) = \left(\frac{\rho}{2\pi}\right)^{d/2} \exp(- \rho \|\theta_*\|_2^2/2) ,
\]
and $\delta=0$. We obtain a regret bound of
 bound of
 \[
O\left( \sqrt{K d T \ln T} + \rho \|\theta_*\|_2^2 \sqrt{K T/d} -
  (\ln \rho) \sqrt{d KT/\ln (d T)}\right) .
\]
By choosing $\rho=1/T^2$, we obtain a bound
 \[
O\left( \sqrt{K d T \ln T} + \frac{\|\theta_*\|_2^2}{T} \sqrt{K/d} \right) .
\]

For linear bandits, where we
have $w(\theta,x)=\theta$ and $K=d$, 
the bound becomes
\[
  O\left( d \sqrt{T \ln T} + \frac{\|\theta_*\|_2^2}{T} \right) ,
\]
which matches that of \cite{dani2008stochastic,abbasi2011improved}, and thus
is not improvable. In comparison, ignoring logarithmic terms, the
previous results for standard Thompson Sampling in
\cite{agrawal2013thompson,abeille2018improved} led to a frequentist
regret bound of
$\tilde{O}(d^{3/2} \sqrt{T})$,
which is inferior to what we obtain here for \algname\ Thompson
Sampling.

\section{Simulation Study}
\label{sec:simulation}

We use a numerical example to show that the algorithm considered in
this paper can be implemented using standard MCMC sampling techniques. Moreover under appropriate conditions,
\algname\ Thompson Sampling can indeed lead to better regret than standard
Thompson Sampling. This verifies the theoretical analysis.
In this example, we consider a simple non-contextual linear bandit problem,
with $d=100$. Let $\theta_*=[1,1,0,0,\ldots]$ be the optimal
parameter, $a_*=[1,0,0,\ldots]$ be the optimal arm, and $\cA = \{a_*\}
\cup \{[0,a']\}$ with $a' \in \rR^{d-1}$ and $\|a'\|_2=0.2$.
We consider Gaussian prior $p_0(\theta) \sim \exp(-\rho
\|\theta\|_2^2/2)$ with $\rho=100$. If we draw $\theta$ from this
prior, $\rE \|\theta\|_2^2=1$, which is consistent with the fact that
$\|\theta_*\|_2^2=2$. For each $a_t \in \cA$, the observation $r_t$ is generated by
adding a uniform random noise from $[-0.5,0.5]$ to $\theta_*^\top
a_t$.
We implemented \algname\ Thompson Sampling algorithm with $\lambda \in
\{0, 0.01, 0.1, 1\}$, where $\lambda=0$ corresponds to the standard
Thompson Sampling. We set $\eta=1$ in this example, which appears to
be an appropriate choice for this problem. 
For simplicity, we set $b=\infty$ in \eqref{eq:loss}, and run the experiments for $100$
times. We then plot the average regret versus time $t$ in Figure~\ref{fig:ts}. It shows
that for this example, there is a benefit of using the \priorname\ exploration.

Our implementation of the \algname\ Thompson Sampling method employs stochastic gradient Langevin
dynamics (SGLD) \cite{welling2011bayesian}. At each time step $t$, we
select a data point $i \in [t]$ uniformly at random, and use
the following stochastic gradient update rule:
\[
\tilde{\theta} \leftarrow \tilde{\theta} - \delta_t
[\nabla_\theta L(\tilde{\theta},x_i,a_i)- t^{-1}\ln p_0(\tilde{\theta})]
+ \sqrt{2\delta_t/t} \; \epsilon_i ,
\]
with $\epsilon_i \sim N(0,I)$. Here we run $t$ random SGLD updates at each time
$t$, with a fixed learning rate $\delta_t=0.01$.  

\begin{figure}
  \centering
  \includegraphics[width=0.4\textwidth]{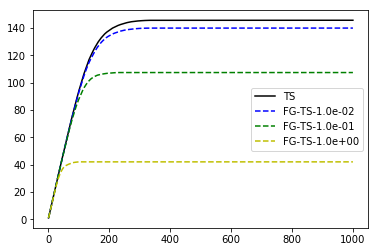}
  \caption{Comparison of \algname\ Thompson Sampling with different
    $\lambda$, where the legend FG-TS-$\lambda$ denotes \algname\
    Thompson Sampling with parameter $\lambda$.
  }
  \label{fig:ts}
\end{figure}

\section{Generalization to Reinforcement Learning}
 
 We consider a simple extension of our analysis to 
 contextual episodic Markov decision process (MDP) with unknown but deterministic transitions, denoted by $M= \mathrm{MDP}(\cS,\cA,H,P,r)$.
Similar contextual MDP models were also studied recently in \cite{jiang2017contextual,CNAAJR18-neurips}.
 Here  $\cS$ and $\cA$ are state and action spaces. The number $H$ is the
 length of each episode. $P=\{p^h\}_{h=1}^H$ and $r=\{r^h\}_{h=1}^H$
 are the state transition probability measures and the random
 rewards. In this work, we assume that the transition probability is
 deterministic (but unknown), and leave the general case to future work. 
 
 The player interacts with this contextual episodic MDP as follows. 
 In each episode $t=1,\ldots,T$, 
 \begin{itemize}
 \item A context $x_t^1 \in \cS^1 \subset \cS$ is picked arbitrarily by an adversary. 
 \item At each step $h =1,2,\ldots,H$
 \begin{itemize}
 \item The player observes the state $x_t^h \in \cS^h \subset \cS$
 \item The player picks a valid action $a_t^h \in \cA$ (we allow a subset of actions in $\cA$ to be valid for each state in $\cS^h$)
 \item The player receives a random reward $r_t^h \in [0,1]$
 \item The player reaches a new state $x_t^{h+1} \in \cS$  deterministically.
 \end{itemize}
 \item The episode terminates after $h=H$ steps.
 \end{itemize}
 

 The goal of MDP is to optimize the expected cumulative rewards:
 \[
 \rE \sum_{s=1}^T \sum_{h=1}^H r_t^h .
\]
It is known that the optimal policy can be derived from the $Q$
function of the MDP, which is denoted by
 $Q^h(x^h,a^h)$ at step $h$ ($h=1,\ldots,H$). It satisfies the Bellman's equation:
 \[
   Q^h(x^h,a^h) = \rE \; [r^h|x^h,a^h] + \max_{a^{h+1}}
   Q^{h+1}(x^{h+1},a^{h+1}) ,
 \]
 where at each state $x^h$ with action $a^h$, we observe reward $r^h$
 and transit to state $x^{h+1}$ deterministically.
 For simplicity, we assume that $Q^{H+1}(\cdot)=0$. 
 We also use the convention
 \[
   V^h(x^h) = \max_{a^h} Q^h(x^h,a^h) .
 \]

 The regret of an MDP algorithm at each time step $t$ is defined as:
 \[
   \reg_t = V^1(x_t^1) - \rE \; \sum_{h=1}^H r_t^h .
 \]

 Consider a set of function classes $\cF \subset \cF^1 \times \cdots
 \times \cF^H \times \cF^{H+1}$, and for each $f \in \cF$, $f^h \in \cF^h: \cS^h \times
 \cA \to \rR$. For notation simplicity, we assume that
 $f^{H+1}(\cdot)=0$, and in order to simplify the notations, we do not introduce a parameter
 $\theta$ into the function definition.  Moreover, we assume that
 $\cS^h$ contains the information of $\cS^1$ so that we can recover
 $x^1 = x^1(x^h) \in \cS^1$  from $x^h \in \cS^h$. This can always be
 made possible by concatenating $[x^1,x^h]$ and regard the result as
 $x^h$. 

We can also define
 \begin{align*}
  a^h(f,x^h) \in  \arg\max_a f^h(x^h,a) , \qquad
 f^h(x^h)=  \max_a f^h(x^h,a) ,
 \end{align*}
and we can  introduce the Bellman operator $\cT$:
 \[
   [\cT  f]^h (x^h,a^h) = \rE [ r^h | x^h, a^h] + f^{h+1}(x^{h+1})  ,
 \]
 where $r^h$ is the observed reward at $(x^h,a^h)$, and $x^{h+1}$ is
 the next state.

\subsection{Thompson Sampling}

 Consider one episode with context $x_t^1$, and sequence $[x_t,a_t,r_t]=\{[x_t^h,a_t^h,r_t^h]\}_{h=1}^H$ which is appropriated generated.
  To apply Thompson Sampling, we define for $f \in \cF$ and for $h= 1, 2,\ldots, H$:
 \begin{align*}
 L^h(f,x_t,a_t,r_t) =& \eta ( f^h(x_t^h,a_t^h) - (r_t^h+ f^{h+1}(x_t^{h+1})))^2 ,
  \end{align*}
 and we let 
 \[
 L^0(f,x_t,a_t,r_t) = -\lambda f^1(x_t^1) .
 \]
 Here $L^0$ is the data dependent \priorname\ exploration prior term, which encourages high quality exploration for reinforcement learning. 
 We note that we do not have to incorporate $f^h(\cdot)$ for $h>1$ because $f^1(\cdot)$ denotes the overall value function of the model, and thus is sufficient for our purpose. 
 
Let the history $S_{t-1}=[(x_{1},a_1,r_1),\ldots,(x_{t-1},a_{t-1},r_{t-1})]$. 
 At episode $t$, we define the posterior of $f$ as
 \begin{equation}
 p(f|S_{t-1}) =
 \frac{p_0(f)
 \exp \left(- \sum_{s=1}^{t-1} \sum_{h=0}^H L^h(f,x_s,a_s,r_s) \right)}
 {\rE_{f \sim p_0(f)}
 \exp \left(- \sum_{s=1}^{t-1}  \sum_{h=0}^H L^h(f,x_s,a_s,r_s) \right)}
 . \label{eq:posterior-rl}
 \end{equation}
 
 Given $f_t$ drawn from the posterior, the Thompson Sampling method employs the associated greedy policy $\pi(f_t)$ as: at step $h$ and state $x^h$, we pick an action that maximizes the value according to the model
 \begin{equation}
 \pi(f_t): \qquad a^h = a^h(f_t,x^h) . \label{eq:rl-greedy-policy}
 \end{equation}
 This policy, conditioned on the context $x_t^1 \in \cX$, induces a distribution on $[x_t,a_t,r_t]$, which we denote as 
 \[
 p(\cdot|f_t,x_t^1) :  \quad \text{distribution induced by $\pi(f_t)$} .
 \]
 The Thompson Sampling algorithm for RL is given in Algorithm~\ref{alg:ts-rl}.
 
 \begin{algorithm}[H]
\caption{Thompson Sampling for Contextual MDP}
\label{alg:ts-rl}
  \begin{algorithmic}[1]
  \FOR{$t=1,2,\ldots,T$}
  \STATE  Observe $x_t^1 \in \cS$
  \STATE  Draw $f_t \sim p(\cdot|S_{t-1})$ according to \eqref{eq:posterior-rl} 
  \STATE  Play episode $t$ using the greedy policy $\pi(f_t)$ of \eqref{eq:rl-greedy-policy}
 \STATE   Observe trajectory $[x_t,a_t,r_t]$ 
  \ENDFOR
\end{algorithmic}
\end{algorithm}

 \subsection{Regret Analysis}

 We have the following assumption.
 \begin{assumption}[Realizability]
   Assume that $Q \in \cF$.
   \label{assump:realizable}
 \end{assumption}

 The following assumption extends \eqref{eq:bandit-linear} for
 contextual bandits. It is a relatively strong assumption, but similar
 assumptions were also 
 required in related works such as \cite{jiang2017contextual,jin2020provably}.
\begin{assumption}
A contextual MDP $M$ is linearly embeddable, if given any context $x^1 \in \cX$, and $h \in [H]$, 
we have the following representation for $f \in \cF$:
\begin{align}
f^h(x^h,a^h)-   [\cT f]^h(x^h,a^h) =& w^h(f,x^1(x^h))^\top \phi^h(x^h,a^h) , \label{eq:rl-linear}
\end{align}
where $\phi^h(x^h,a^h) \in \rR^K$.
\label{assump:mdp-linear}
\end{assumption}

This decomposition can be regarded as a context and action dependent version of the Bellman factorization and Bellman rank in \cite{jiang2017contextual}, and our definition can naturally handle infinite action spaces.  
The linearly embeddable condition generalizes the linear MDP model of
\cite{jin2020provably}, where we allow contextually dependent weights
that can be a nonlinear function with unknown embedding to be learned.
For simplicity, in the analysis, we will avoid dealing with range conditions for RL
by imposing the following conditions directly. Alternatively, we may also
employ the truncation technique used in our bandit analysis to handle out
of range function values.
\begin{assumption}
We assume that there exists $b \geq 0$ so that for all $f \in \cF$:
\[
f^1(x^1)  \in [0, H b] .
\]
Moreover, we assume that for all $h \geq 1$:
\[
f^h(x^h,a^h) - [\cT f]^h (x^h,a^h) \in [-b,b] .
\]
\label{assump:rl-value-range}
\end{assumption}

 Next, we will use the following key observation, referred to as the value-function error decomposition in \cite{jiang2017contextual}.
 Given any $f_t \in\cF$ and $x_t^1$. Let $[x_t,a_t,r_t]\sim p(\cdot|f_t,x_t^1)$ be the trajectory of the greedy policy $\pi(f_t)$, we have
  \begin{align}
  \reg_t 
   =
  \rE_{[x_t,a_t,r_t]\sim p(\cdot|f_t,x_t^1)}
 \sum_{h=1}^H  \underbrace{f_t^h(x_t^h,a_t^h)-[\cT f]^h(x_t^h,a_t^h)}_{\BE_t^h}
  - \underbrace{\left[ f^1(x^1) - V^1(x^1)\right]}_{\FG_t} , \label{eq:reg-rl-decomp}
  \end{align}
  where
  \[
    \reg_t= V^1(x^1) - \rE_{[x_t,a_t,r_t] \sim p(\cdot|f_t,x_t^1)}
    \left[\sum_{h=1}^H r_t^h\right] .
  \]
 With the above decomposition, we may introduce the decoupling
 coefficient for MDP below, which generalizes Definition~\ref{def:dc}
 for contextual bandits. For simplicity, we only consider the case
 of $b=\infty$ in Definition~\ref{def:dc}. 
  \begin{definition}[Decoupling Coefficient]
 Consider a contextual MDP $M$. Given any $q(f)$ on $\cF$. Let
 \[
 \pi_q(x^h,a^h|x^1)=\rE_{f \sim q(\cdot)} p(x^h,a^h|f,x^0) .
 \]
 Then $\dc(\cF,M)$ is the smallest quantity $K$ so that
 for all $h \geq 1$:
 \begin{align*}
 &\rE_{f \sim q(f)}
 \rE_{[x^h,a^h]\sim p(\cdot|f,x^1)}
 [f^h(x^h,a^h)- [\cT f]^h(x^h,a^h)]\\
 \leq& \inf_{\mu>0} \left[\mu \rE_{[x^h,a^h]\sim \pi_q(\cdot|x^1)}
 \rE_{\tilde{f} \sim q(\cdot)}
  ( \tilde{f}^h(x^h,a^h)- [\cT \tilde{f}]^h (x^h,a^h))^2 +
       \frac{K}{4\mu}\right] .
 \end{align*}
 \label{def:dc-rl}
\end{definition}

Using Definition~\ref{def:dc-rl}, we can obtain the following  regret bound
 from \eqref{eq:reg-rl-decomp}:
\begin{align}
  \rE \reg_t \leq & \mu \rE \; \rE_{\tilde{f} \sim p(\cdot|S_{t-1})} \sum_{h=1}^H  \LS_t^h + \frac{\dc(\cF,M) H}{4 \mu} - \rE
                    \; \FG_t , \label{eq:rl-regret}
 \end{align}
 where
 \[
   \LS_t^h = ( \tilde{f}^h(x_t^h,a_t^h)- [\cT \tilde{f}]^h (x_t^h,a_t^h))^2  .
 \]
 Note that \eqref{eq:rl-regret} is an analogy of \eqref{eq:bandit-regret}.
 The following lemma is a generalization of Lemma~\ref{lem:decouple-linear} for contextual bandits. The proof is essentially the same. 
 \begin{lemma}
 Assume that the linear embedding in Assumption~\ref{assump:mdp-linear} holds for the contextual MDP $M$.
 Then $\dc(\Omega,M) \leq K$. 
  \label{lem:rl-decouple}
 \end{lemma}
 
Using the above definition, we can use the same proof technique as that of contextual bandits to
obtain a regret bound for reinforcement learning. The proof is given in the \suppl.

\begin{theorem}
  Consider Algorithm~\ref{alg:ts-rl}
  with posterior \eqref{eq:posterior-rl}.
Under Assumption~\ref{assump:realizable},
Assumption~\ref{assump:mdp-linear}, and
Assumption~\ref{assump:rl-value-range}. 
Let $\eta \leq \min(0.25,1/(H b^2))$.
Then we have the following expected regret bound:
\begin{align*}
\sum_{t=1}^T \rE  \; \reg_t
\leq&
\frac{\lambda TH}{\eta}  \dc(\cF,M)
      + (0.25 \eta/\lambda) \sum_{t=1}^T \sum_{h=1}^H \; \rE \; \LS_t^h - \sum_{t=1}^T \rE \; \FG_t \\
\leq&  \frac{\lambda TH}{\eta} \dc(\cF,M)
+1.5 \lambda H^2 b^2 T
- \frac{Z_T }{\lambda} ,
\end{align*}
where
\begin{align}
&Z_t  
  = \rE_{S_t} \ln \rE_{\tilde{f} \sim p_0} 
                \exp \left(- \sum_{s=1}^t \sum_{h=0}^H \Delta L^h(\tilde{f},x_s,a_s,r_s)\right)
                \label{eq:rl-log-part}  \\
  &\Delta L^h(f,x,a,r)=\eta [(f^h(x^h,a^h)-(r+ f^{h+1}(x^{h+1}))^2-
    (r- \rE [r|x,a])^2] \quad (h \geq 1) \nonumber \\
  & \Delta L^0(f,x,a,r)  = -\lambda [f(x^1)-V(x^1)] . \nonumber 
\end{align}
\label{thm:regret-rl}
\end{theorem}

To interpret the regret bound, we note that if $H b^2 \geq 4$ and we take $\eta=1/(H
b^2)$, then the bound becomes
\[
\sum_{t=1}^T \reg_t \leq \lambda (K+1.5) H^2 b^2 T
-\frac{1}{\lambda} \rE \ln   \rE_{\tilde{f} \sim p_0(\tilde{f})
 }
 \exp \left(- \sum_{s=1}^{T} \sum_{h=0}^H \Delta L^h(\tilde{f},[x_s,a_s,r_s) \right) .
\]
For $d$ dimensional parametric function class, then, similar to
Section~\ref{sec:parametric}, we have
\[
\ln   \rE_{\tilde{f} \sim p_0(\cdot)} \exp \left(- \sum_{s=1}^{T}
    \sum_{h=0}^H \Delta L^h(\tilde{f},x_s,a_s,r_s) \right)
  = O(-\ln p_0(Q) + H d \ln (H d T)) ,
\]
where $Q \in \cF$ is the true value function.
By taking
\[
  \lambda = \sqrt{\frac{d \ln (HdT)}{(K+2)H b^2 T}} ,
\]
we obtain
\[
  \sum_{t=1}^T \reg_t = O \left( \sqrt{ d (K+2) H^3 b^2 T \ln (H d T)}
      - \ln p(Q) \sqrt{(K+2)H b^2 T/(d \ln (Hd T))}\right) ,
  \]
which is similar to the contextual bandit case, and similar to results of \cite{jin2020provably} for linear MDPs. 

 \section{Conclusion}
 
 This paper presents a general analysis of Thompson Sampling. Contrary to the conventional thinking that the random sampling of Thompson Sampling leads to sufficient exploration, we showed that the standard choice of likelihood function in Thompson Sampling can be suboptimal due to the lack of aggressiveness to encourage optimistic exploration. 
  To remedy this problem, we proposed a modification of Thompson Sampling with an additional \priorname\ exploration term. The resulting method can be viewed as an implementation of the general optimism in the face of uncertainty principle for Thompson Sampling. 
  It was shown that this method led to minimax optimal regret bound for the general contextual bandit problem with finite actions.  
 Moreover, we extended the analysis to handle infinite actions when the action space is linearly embeddable, with regret bound matching known lower bounds. 
 
 We also demonstrated that this new theoretical framework for Thompson
 Sampling can be extended to the reinforcement learning setting. Our
 analysis of \priorname\ exploration employs  a new proof technique
 using decoupling coefficient to handle exploration. It reduces the
 online regret bound analysis into an online least squares estimation
 problem, which in spirit is similar to \cite{foster2020beyond}. We
 then bound the online least squares loss using aggregation
 techniques. In comparison, the technique of  \cite{foster2020beyond}
 cannot be directly generalized to handle reinforcement learning, as
 noted by the authors there. 
While this paper considers the simple case of unknown deterministic transition
dynamics
for reinforcement learning, more general situation with random
transitions can also be handled. We leave detailed studies to future work. 

As pointed out in Section~\ref{sec:related}, bounds obtained in this
paper may not be optimal for certain structured bandit problems. It
will be interesting to explore whether such structures can be
incorporated into our analysis to improve the resulting bounds. 



\appendix
\section{Proof of Theorem~\ref{thm:bandit}}
\label{apx:proof-bandit}.

We use the following estimate, which directly follows from the
sub-Gaussian definition of noise.
\begin{lemma}
Consider $L(\cdot)$ defined in \eqref{eq:loss}.
If $\eta (b+1)^2\leq 1$, then 
\begin{align*}
&\ln \rE_{\tilde{\theta} \sim p(\cdot|x_t,S_{t-1})} \rE_{r_t|x_t,a_t} \;
\exp(-\Delta L(\tilde{\theta},x_t,a_t,r_t)) \\
\leq&  - 0.25 \eta \rE_{\tilde{\theta} \sim p(\cdot|x_t,S_{t-1})} 
      \LS_t + \lambda \rE_{\theta_t \sim p(\cdot|x_t,S_{t-1})} \FG_t    + 1.5 \lambda^2 (b+1)^2 .
\end{align*}
\label{lem:exp}
\end{lemma}
\begin{proof}
  Let $\epsilon_t = r_t - f_*(x_t,a_t)$.
  Using Assumption~\ref{assump:reward}, we have
  \begin{align}
    \rE_{r_t|x_t,a_t} \exp( -2 \eta
    \epsilon_t(f_*(x_t,a_t)-f(\tilde{\theta},x_t,a_t)) )
    \leq \exp(0.5 \eta^2 \LS_t') . \label{eq:lem-exp-proof-1}
  \end{align}
  Therefore
  \begin{align*}
    &-\Delta L(\tilde{\theta},x_t,a_t,r_t) \\
    = &
    -\eta (\epsilon_t+f_*(x_t,a_t)-f(\tilde{\theta},x_t,a_t))^2 + \eta
    (\epsilon_t)^2+ \lambda [\min(b,f(\tilde{\theta},x_t))-f_*(x_t)]\\
    \leq &
       -2 \eta \epsilon_t(f_*(x_t,a_t)-f(\tilde{\theta},x_t,a_t)) - \eta \LS_t' +
       \lambda  \FG_t' ,
  \end{align*}
  where
  \begin{align*}
    \LS_t'= (f(\tilde{\theta},x_t,a_t)-f_*(x_t,a_t))^2 ,
\quad    \FG_t' = f_b(\tilde{\theta},x_t)-f_*(x_t) .
  \end{align*}
  
  Now using \eqref{eq:lem-exp-proof-1}, we obtain:
  \[
    \rE_{r_t |x_t , a_t}\; 
    \exp(-\Delta L(\tilde{\theta},x_t,a_t,r_t)) 
    \leq \exp (  - (1-0.5 \eta) \eta \LS_t' + \lambda \FG_t') .
  \]
  Therefore
  \begin{align}
   & \ln \rE_{\tilde{\theta} \sim p(\cdot|x_t,S_{t-1})} \; \rE_{r_t|x_t,a_t} \; \exp(-\Delta L(\tilde{\theta},x_t,a_t,r_t)) \nonumber\\
    \leq& \ln \rE_{\tilde{\theta} \sim p(\cdot|x_t,S_{t-1})}  \; \exp
          ( -(1-0.5\eta)\eta \LS_t' + \lambda \FG_t')\nonumber
    \\
    \leq& \frac{2}{3} \ln \rE_{\tilde{\theta} \sim p(\cdot|x_t,S_{t-1})} \; \exp ( -1.5(1-0.5\eta)\eta \LS_t') +\frac{1}{3} \ln
          \rE_{\tilde{\theta}\sim p(\cdot|x_t,S_{t-1})} \; \exp(3 \lambda
          \FG_t') . \label{eq:lem-exp-proof-2}
  \end{align}
  The last inequality follows from $\rE_{Z_1,Z_2} \exp(Z_1 +
  Z_2) \leq \rE_{Z_1}^{2/3} \exp(1.5 Z_1) \rE_{Z_2}^{1/3} \exp(3
  Z_2)$,  which follows from the H\"{o}lder's inequality.
  
  Observe that $\psi(z) = (e^z-z-1)/z^2$ is an increasing function
  in $z$, $\psi(0) =0.5$, we have
  $\exp(z)-1 \leq z + 0.5 z^2$ for $z \leq 0$. Therefore
  \begin{align*}
&   \rE_{\tilde{\theta}} \; \exp ( -1.5(1-0.5\eta)\eta \LS_t) -1\\
    \leq &  -1.5(1-0.5\eta)\eta \rE_{\tilde{\theta}}  \LS_t + 1.5
           \cdot 0.75 (1-0.5\eta)^2\eta^2
   \rE_{\tilde{\theta}}  \LS_t^2\\
    \leq & -1.5(1-0.5\eta) \eta (1-0.75(1-0.5\eta)(b+1)^2\eta)  \rE_{\tilde{\theta}} \LS_t
      \leq - 0.375 \eta \rE_{\tilde{\theta}} \LS_t .
  \end{align*}
  The second inequality used $\LS_t \leq (b+1)^2$. The last inequality
  used $\eta \leq 1/(1+b)^2 \leq 1/4$, and thus
$(1-0.5\eta)  (1-0.75(1-0.5\eta)(b+1)^2\eta)   \geq
(1-0.5\eta)(1-0.75(1-0.5\eta)) \geq 0.25$.
Now using $\ln z \leq z-1$, and $-\LS_t' \leq -\LS_t$, we obtain
  \begin{align}
   & \frac23 \ln \rE_{\tilde{\theta}} \; \exp ( -1.5(1-0.5\eta)\eta \LS_t')\nonumber\\
    \leq& \frac23 \left[ \rE_{\tilde{\theta}} \; \exp ( -1.5(1-0.5\eta)\eta \LS_t')
      -1 \right]\nonumber\\
        \leq& \frac23 \left[ \rE_{\tilde{\theta}} \; \exp ( -1.5(1-0.5\eta)\eta \LS_t)
      -1 \right]
    \leq - 0.25 \eta \rE_{\tilde{\theta}} \LS_t . \label{eq:lem-exp-proof-3}
  \end{align}
  Moreover, since $|\FG_t'| \leq b+1$, we obtain from Chernoff bound that
  \[
    \frac{1}{3} \ln
    \rE_{\tilde{\theta}} \; \exp(3 \lambda \FG_t') \leq
    \lambda \rE_{\tilde{\theta}}  \; \FG_t' + 1.5 \lambda^2 (b+1)^2
   =     \lambda \rE_{\theta_t \sim p(\cdot|x_t,S_{t-1})} \; \FG_t + 1.5
   \lambda^2 (b+1)^2 .
 \]
 Plug this inequality and \eqref{eq:lem-exp-proof-3} into
 \eqref{eq:lem-exp-proof-2}, 
  we obtain the desired bound.
\end{proof}

The following lemma is standard in the analysis of online aggregation methods.
\begin{lemma}
We have
\[
  (0.25 \eta/\lambda) \rE \LS_t
  - \rE \FG_t
  \leq  1.5 \lambda (b+1)^2 + \lambda^{-1} ( Z_{t-1}-Z_t) .
\]
\label{lem:one-step}
\end{lemma}
\begin{proof}
Define
\[
W_t(\theta|S_t) =  \exp \left(- \sum_{s=1}^t \Delta
  L(\theta,x_s,a_s,r_s)\right) I(\theta \in \Omega_t) ,
\]
then
\[
Z_t = \rE_{S_t} \ln \rE_{\tilde{\theta} \sim p_0}  \; W_t(\tilde{\theta}|S_t)  .
\]
Let $\Omega_0=\Omega$. It follows from $\Omega_t \subset \Omega_{t-1}$ that 
\[
p(\tilde{\theta} |x_t,S_{t-1}) = \frac{W_{t-1}(\tilde{\theta} |S_{t-1}) I(\tilde{\theta} \in\Omega_t)
}{\rE_{\tilde{\theta}
    \sim p_0} W_{t-1}(\tilde{\theta} |S_{t-1})I(\tilde{\theta} \in\Omega_t)} .
\]
We have
\begin{align*}
     Z_t =& Z_{t-1}+ \rE_{S_t} \ln 
    \frac{\rE_{\tilde{\theta} \sim p_0}  \; W_t(\tilde{\theta}|S_t)}{\rE_{\tilde{\theta} \sim
            p_0}  \; W_{t-1}(\tilde{\theta}|S_{t-1})}\\
  \leq& Z_{t-1}+ \rE_{S_t} \ln 
    \frac{\rE_{\tilde{\theta} \sim p_0}  \;
        W_t(\tilde{\theta}|S_t)}{\rE_{\tilde{\theta} \sim p_0}  \;
        W_{t-1}(\tilde{\theta}|S_{t-1}) I(\tilde{\theta}\in\Omega_t)}\\
  =& Z_{t-1}+ \rE_{S_t} \ln 
    \frac{\rE_{\tilde{\theta} \sim p_0}  \;
     W_{t-1}(\tilde{\theta}|S_{t-1}) I(\tilde{\theta}\in\Omega_t) \exp(-\Delta L(\tilde{\theta},x_t,a_t,r_t))}{\rE_{\tilde{\theta} \sim p_0}  \; W_{t-1}(\tilde{\theta}|S_{t-1})I(\tilde{\theta}\in\Omega_t)}\\
    =&Z_{t-1}+ \rE_{S_{t}} \ln 
    \rE_{\tilde{\theta} \sim p(\cdot|x_t,S_{t-1})}  \; \exp(-\Delta L(\tilde{\theta},x_t,a_t,r_t))  \\
  \stackrel{(a)}{\leq}&Z_{t-1}+
\rE_{S_{t-1},x_t,a_t}                        \ln 
    \rE_{r_t|x_t,a_t}\;\rE_{\tilde{\theta} \sim p(\cdot|x_t,S_{t-1})} \; \exp(-\Delta L(\tilde{\theta},x_t,a_t,r_t)) \\
  \stackrel{(b)}{\leq}&  Z_{t-1} - 0.25 \eta \rE \; \LS_t + \lambda
                        \rE\; \FG_t + 1.5 \lambda^2 (b+1)^2 .
\end{align*}
In the above derivation, $(a)$ used the Jensen's inequality and the
concavity of $\ln z$;  $(b)$ used Lemma~\ref{lem:exp}. This proves the
lemma. 
\end{proof}

We are now ready to prove Theorem~\ref{thm:bandit}.
Using \eqref{eq:bandit-regret} with $\mu=0.25\eta/\lambda$, we obtain
the first inequality of the theorem.
Notice that $Z_0=0$.
We can now apply Lemma~\ref{lem:one-step} and sum over $t=1$ to $t=T$
to obtain the second inequality. 

\section{Proof of Lemma~\ref{lem:decouple-linear}}
\label{apx:proof-decouple-linear}

First, we prove the case with $b=\infty$. 
 Consider any $q(\theta)$, $x \in \cX$, and $\mu>0$.
 Define
 \[
 \Sigma = \rE_{a \sim \pi_q(a|x)} \phi(x,a) \phi(x,a)^\top .
 \]
 Let $\xi_j$ ($j=1,\ldots,K$) be an orthonormal basis of eigenvectors of $\Sigma$. 
 It follows that
 \begin{align}
 &\rE_{a \sim \pi_q(a|x)} 
 \rE_{\theta \sim q(\theta)}
 (f(\theta,x,a)-f_*(x,a))^2 \\
 =& \rE_{a \sim \pi_q(a|x)} 
 \rE_{\theta \sim q(\theta)}
 \left(\sum_{j=1}^K (w(\theta,x)-w_*(x))^\top \xi_j) (\phi(x,a)^\top \xi_j)\right)^2\nonumber\\
 =& \sum_{i=1}^K \sum_{j=1}^K \rE_{a \sim \pi_q(a|x)} (\phi(x,a)^\top \xi_i)(\phi(x,a)^\top \xi_j) \nonumber\\
 &\quad \cdot \rE_{\theta \sim q(\theta)} ((w(\theta,x)-w_*(x))^\top \xi_i)((w(\theta,x)-w_*(x))^\top \xi_j)\nonumber\\
\stackrel{(a)}{=}& \sum_{j=1}^K \rE_{a \sim \pi_q(a|x)} (\phi(x,a)^\top \xi_j)^2
 \rE_{\theta \sim q(\theta)} ((w(\theta,x)-w_*(x))^\top \xi_j)^2 \nonumber\\
 =& \sum_{j=1}^K q_j \rE_{a \sim \pi_q(a|x)} (\phi(x,a)^\top \xi_j)^2 \nonumber\\
 =& \sum_{j=1}^K q_j \rE_{\theta \sim q(\theta)} (\phi(x,a(\theta,x))^\top \xi_j)^2 , \label{eq:decouple-quadratic}
  \end{align}
  where $(a)$ used the fact that $\rE_{a \sim \pi_q(a|x)}
  (\phi(x,a)^\top \xi_i)(\phi(x,a)^\top \xi_j)=0$ when $i \neq j$,
  and we let
 \begin{equation}
 q_j = \rE_{\theta \sim q(\theta)} ((w(\theta,x_t)-w_*(x_t))^\top \xi_j)^2 .
 \label{eq:decouple-qj}
 \end{equation}
 We have
 \begin{align*}
 &\rE_{\theta \sim q(\theta)}
 |f(\theta,x,a(\theta,x))-f_*(x,a(\theta,x))|\\
 =&\rE_{\theta \sim q(\theta)} \; \left| \sum_{j=1}^K ((w(\theta,x)-w_*(x))^\top \xi_j) (\phi(x,a(\theta,x)^\top \xi_j)\right| \\
 \leq &\rE_{\theta \sim q(\theta)} \; 
 \left[ (4 \mu)^{-1} \sum_{j=1}^K q_j^{-1} ((w(\theta,x_t)-w_*(x_t))^\top \xi_j)^2
 + \mu \sum_{j=1}^K q_j (\phi(x_t,a(\theta,x_t)^\top \xi_j)^2\right] ,
 \end{align*}
 where the last inequality used Young's inequality for products.
We can obtain the following bound by using \eqref{eq:decouple-qj} to simplify the first term, and  \eqref{eq:decouple-quadratic} to simply the second term.
\begin{align*}
 &\rE_{\theta \sim q(\theta)}
 |f(\theta,x,a(\theta,x))-f_*(x,a(\theta,x))|\\
 \leq& \frac{K}{4\mu} +\mu\rE_{a \sim \pi_q(a|x)} 
 \rE_{\theta \sim q(\theta)}
 (f(\theta,x,a)-f_*(x,a))^2 .
 \end{align*}
 This implies the desired inequality in Definition~\ref{def:dc} with $b=\infty$.
 
For finite $b \geq 1$, we consider any $q(\theta)$, $x \in \cX$.
 We define
 \[
 \tilde{w}(\theta,x)
 =\begin{cases}
 w(\theta,x) & f(\theta,x,a(\theta,x)) \leq b \\
 w_*(x)+ \beta(\theta) (w(\theta,x)-w_*(x))  
 & f(\theta,x,a(\theta,x)) > b ,
 \end{cases}
 \]
 where
 \[
 \beta(\theta) = \frac{b-f_*(x,a(\theta,x))}
 {f(\theta,x,a(\theta,x))-f_*(x,a(\theta,x))} \in [0,1].
 \]
 Let 
 \[
 \tilde{f}(\theta,x,a) = \tilde{w}(\theta,x)^\top \phi(x,a) .
 \]
 Then $\tilde{f}(\theta,x,a(\theta,x))=f_b(\theta,x,a(\theta,x))$, 
 and $|\tilde{f}(\theta,x,a')-f_*(x,a')|\leq |f_b(\theta,x,a')-f_*(x,a')|$.
 We have
 \begin{align*}
 &\rE_{\theta \sim q(\theta)}
 |f_b(\theta,x_t,a(\theta,x))-f_*(x,a(\theta,x))| \\
 =&\rE_{\theta \sim q(\theta)}
 |\tilde{f}(\theta,x,a(\theta,x))-f_*(x,a(\theta,x))|\\
 \leq& \mu \rE_{a \sim \pi_q(a|x)} 
 \rE_{\theta \sim q(\theta)}
 (\tilde{f}(\theta,x,a)-f_*(x,a))^2 +\frac{K}{4\mu}\\
 \leq&  \mu\rE_{a \sim \pi_q(a|x)} 
 \rE_{\theta \sim q(\theta)}
 (f_b(\theta,x,a)-f_*(x,a))^2 +\frac{K}{4\mu},
 \end{align*}
 where the first inequality is due to Lemma~\ref{lem:decouple-linear}
 with $b=\infty$ applied to $\tilde{f}$. This proves the desired result.



\section{Proof of Theorem~\ref{thm:regret-rl}}

\begin{lemma}
  If $\eta \max(4,H {b}^2) \leq 1$, then
  \begin{align*}
    & \ln \rE_{\tilde{f} \sim p(\cdot|S_{t-1})} \rE_{[x_t,r_t,a_t] \sim p(\cdot|\tilde{f},x_t^1)}\; \exp\left(-\sum_{h=0}^{H} \Delta
    L^h(\tilde{f},x_t,a_t,r_t))   \right) \\
    \leq & - 0.25 \eta \rE_{\tilde{f} \sim p(\cdot|S_{t-1})}
           \sum_{h=1}^H \LS_t^h
           + \lambda \rE_{f_t \sim p(\cdot|S_{t-1})} \FG_t + 1.5
           \lambda^2 H^2 b^2 .
  \end{align*}
  \label{lem:rl-exp}
\end{lemma}
\begin{proof}
  Let $\epsilon_t^h = r_t^h - \rE [r_t^h|x_t^h,a_t^h]$ be the noise,
  and
\[
  \Delta \tilde{f}^h(x_t^h,a^h)= \tilde{f}^h(x_t^h,a^h)-[\cT
  \tilde{f}]^h  (x^h,a^h)
  \]
be  the Bellman residual. 

Since $r_t^h \in [0,1]$, we obtain from
  Chernoff bound that
  \[
    \rE_{r_t^h|x_t^h,a_t^h} \exp(\rho \epsilon_t^h) \leq
    \exp(\rho^2/8) .
  \]
  Since for $h \geq 1$: after some algebraic manipulations, we can get
  \[
  \Delta L^h(\tilde{f},x_t,a_t,r_t)= \eta (\Delta \tilde{f}^h(x_t^h,a_t^h))^2
  -2 \eta  \epsilon_t^h \Delta \tilde{f}^h(x_t^h,a_t^h) ,
\]
we have for $h \geq 1$:
\begin{align}
  \rE_{r_t^h|x_t^h,a_t^h} \exp(-\Delta L^h(\tilde{f},x_t,a_t,r_t)) \leq&
  \exp(-\eta(1-0.5\eta)(\Delta \tilde{f}^h(x_t^h,a_t^h))^2)\nonumber\\
  =&  \exp(-\eta(1-0.5\eta) \LS_t^h) .
  \label{eq:lem-rl-exp-proof-1}
\end{align}
Since given $\tilde{f}$, $(x_t,a_t)$ is a deterministic sequence of $x_t^1$,  and the
only randomness is from $r_t$, it follows that
\begin{align*}
  &\rE_{[r_t,x_t,a_t] \sim (p(\cdot|\tilde{f},x_t^1)} \exp \left(-\sum_{h=0}^H \Delta
    L^h(\tilde{f},x_t,a_t,r_t) \right) \\
  \stackrel{(a)}{\leq}& \rE_{[r_t,x_t,a_t] \sim p(\cdot|\tilde{f},x_t^1)} \exp \left(-\sum_{h=0}^{H-1} \Delta
    L^h(\tilde{f},x_t,a_t,r_t))  - \eta(1-0.5\eta) \LS_t^H \right) \\
    \stackrel{(b)}{\leq}& \rE_{[r_t,x_t,a_t] \sim p(\cdot|\tilde{f},x_t^1)} \exp \left(-\sum_{h=0}^{H-2} \Delta
    L^h(\tilde{f},x_t,a_t,r_t))  - \sum_{h=H-1}^H
          \eta(1-0.5\eta) \LS_t^h \right) \\
  & \cdots \\
  \leq& \exp \left( -\Delta L^0(\tilde{f},x_t,a_t,r_t) - \sum_{h=1}^H
        \eta(1-0.5\eta) \LS_t^h\right) ,
\end{align*}
where in the above derivation, we have applied
\eqref{eq:lem-rl-exp-proof-1} with $h=H$ in $(a)$,
\eqref{eq:lem-rl-exp-proof-1} with $h=H-1$ in $(b)$, and so on...

It follows that
\begin{align}
& \ln \rE_{\tilde{f} \sim p(\cdot|S_{t-1})} \rE_{[x_t,r_t,a_t] \sim p(\cdot|\tilde{f},x_t^1)}\; \exp\left(-\sum_{h=0}^{H} \Delta
                 L^h(\tilde{f},x_t,a_t,r_t))   \right)  \nonumber\\
  \leq&  \ln \rE_{\tilde{f} \sim p(\cdot|S_{t-1})}  \; \exp\left( -\Delta L^0(\tilde{f},x_t,a_t,r_t) - \sum_{h=1}^H          \eta(1-0.5\eta)
        \LS_t^h \right) \nonumber\\
  \leq&  \frac{1}{3} \ln \rE_{\tilde{f} \sim p(\cdot|S_{t-1})}  \;
        \exp\left( -3\Delta L^0(\tilde{f},x_t,a_t,r_t)\right)  \nonumber\\
      &  + \frac{2}{3H} \sum_{h=1}^H \ln \rE_{\tilde{f} \sim p(\cdot|S_{t-1})}  \; \exp\left(- 1.5 H     \eta(1-0.5\eta)\LS_t^h
        \right)      ,   \label{eq:lem-rl-exp-proof-2}
\end{align}
where the last inequality is due to the Jensen's inequality applied to
$\ln \rE_{\tilde{f}} \exp(g(\tilde{f}))$ as a convex function of
$g(\cdot)$.

The same argument of \eqref{eq:lem-exp-proof-3} with $(b+1)^2$ replaced
by $\max(4,Hb^2)$ implies that
\begin{align}
   \frac2{3H} \ln \rE_{\tilde{f}} \; \exp ( -1.5(1-0.5\eta)\eta H\LS_t^h)
    \leq - 0.25 \eta \rE_{\tilde{f}} \LS_t^h . \label{eq:lem-rl-exp-proof-3}
\end{align}
  Moreover, since $|\Delta L^0(\tilde{f},x_t,a_t,r_t)| \leq \lambda H b$, we obtain from Chernoff bound that
  \[
    \frac{1}{3} \ln
    \rE_{\tilde{f}} \; (-3 \Delta L^0(\tilde{f},x_t,a_t,r_t)) \leq
     \lambda \rE_{f_t \sim p(\cdot|S_{t-1})} \; \FG_t + 1.5
   \lambda^2 H^2 b^2 .
 \]
 Plug this inequality and \eqref{eq:lem-rl-exp-proof-3} into
 \eqref{eq:lem-rl-exp-proof-2}, we obtain the desired bound.
\end{proof}

The following lemma is similar to Lemma~\ref{lem:one-step}. The proof
is identical. 
\begin{lemma}
Consider 
$Z_t$ defined in 
\eqref{eq:rl-log-part}. We have
\[
  (0.25 \eta/\lambda) \sum_{h=1}^H \rE \; \LS_t^h
  - \rE\; \FG_t
  \leq  1.5 \lambda H^2 b^2 + \lambda^{-1} ( Z_{t-1}-Z_t) .
\]
\label{lem:rl-one-step}
\end{lemma}

  We are now ready to proof Theorem~\ref{thm:regret-rl} as follows.
Using \eqref{eq:rl-regret} with $\mu=0.25\eta/\lambda$, we obtain
the first inequality of the theorem.
Notice that $Z_0=0$.
We can now apply Lemma~\ref{lem:rl-one-step} and sum over $t=1$ to $t=T$
to obtain the second inequality. 



\section*{Acknowledgment}
 
 The author would like to thank Christoph Dann for discussions about related works. 
 
\bibliographystyle{plain}
\bibliography{myrefs}

\end{document}